\documentclass[conference]{IEEEtran}
\usepackage{times}

\usepackage{multicol}
\usepackage[bookmarks=true]{hyperref}
\usepackage{makeidx}         
\usepackage{graphicx}        
\usepackage[bottom,hang,flushmargin]{footmisc}
\usepackage{amsmath}
\usepackage{amsfonts}
\usepackage{amsthm}
\usepackage[hang,tight]{subfigure}
\usepackage[font=footnotesize]{caption}
\usepackage{algorithmic}
\usepackage{algorithm}
\usepackage{color}

\newtheoremstyle{nospace}
{2pt}   				
{2pt}   				
{\itshape}  			
{} 		  		    	
{\bfseries} 			
{.}         			
{5pt plus 1pt minus 1pt}
{}          			

\theoremstyle{nospace} \newtheorem{theorem}{Theorem}
\theoremstyle{nospace} 
\theoremstyle{nospace} 
\theoremstyle{nospace} \newtheorem{remark}{Remark}
\theoremstyle{nospace} 
\theoremstyle{nospace} 
\theoremstyle{nospace} \newtheorem{corollary}{Corollary}
\theoremstyle{nospace} \newtheorem{assumption}{Assumption}

\newcommand{\M}{\mathcal{M}}
\newcommand{\supp}{\text{spt}}
\newcommand{\R}{\mathbb{R}}
\newcommand{\RR}{\mathbb{R}}
\newcommand{\NN}{\mathbb{N}}
\newcommand{\N}{\mathbb{N}}

\newcommand{\Lf}{\mathcal{L}_f}
\newcommand{\Lg}{\mathcal{L}_g}

\newcommand{\vect}{\text{vec}}

\IEEEoverridecommandlockouts

\begin{document}

\title{Technical Report: Convex Optimization of Nonlinear Feedback Controllers via Occupation Measures \thanks{If you want to cite this report, please use the following reference instead:
A. Majumdar, R. Vasudevan, M. M. Tobenkin, and R. Tedrake, ``Convex Optimization of Nonlinear Feedback Controllers via Occupation Measures'', In \textit{Proceedings of Robotics: Science and Systems (RSS)}, 2013. }}


\author{\authorblockN{Anirudha Majumdar,
Ram Vasudevan, Mark M. Tobenkin, and Russ Tedrake}
\authorblockA{Computer Science and Artificial Intelligence Laboratory\\
Massachusetts Institute of Technology\\
Cambridge, MA 02139\\
Email: \{anirudha,ramv,mmt,russt\}@mit.edu}}




\maketitle

\begin{abstract}
In this paper, we present an approach for designing feedback controllers for polynomial systems that maximize the size of the time-limited backwards reachable set (BRS). 
We rely on the notion of \emph{occupation measures} to pose the synthesis problem as an infinite dimensional linear program (LP) and provide finite dimensional approximations of this LP in terms of semidefinite programs (SDPs). The solution to each SDP yields a polynomial control policy and an outer approximation of the largest achievable BRS. In contrast to traditional Lyapunov based approaches which are non-convex and require feasible initialization, our approach is convex and does not require any form of initialization.
The resulting time-varying controllers and approximated reachable sets
are well-suited for use in a trajectory library or feedback motion planning algorithm.
We demonstrate the efficacy and scalability of our approach on five nonlinear systems.
\end{abstract}


\section{Introduction}
\label{sec:introduction}

Dynamic robotic tasks such as flying, running, or walking demand controllers that push hardware platforms to their physical limit while managing input saturation, nonlinear dynamics, and underactuation. Though motion planning algorithms have begun addressing several of these tasks \cite{LaValle06}, the constructed open loop motion plans are typically insufficient due to their inability to correct for deviations from a planned path. Despite the concerted effort of several communities,
the design of feedback control laws for underactuated nonlinear systems with input saturation remains challenging.

Popular techniques for control synthesis rely either on feedback linearization \cite{Sastry99} or on linearizing the dynamics about a nominal operating point in order to make Linear Quadratic Regulator based techniques or Linear Model Predictive Control \cite{Bemporad02a} applicable. Unfortunately, feedback linearization is generally untenable for underactuated systems especially in the presence of actuation limits, and those techniques that rely on linearizations lead to controllers that are valid only locally around the operating point. Dynamic Programming and Hamilton-Jacobi Bellman Equation based techniques \cite{Ding10,Mitchell05a} have also been used for feedback control design. However, these methods suffer from the curse of dimensionality, and can require exorbitant grid resolution for even low dimensional systems \cite{Munos02}.

\subsection{Our Contributions}

In this paper, we attempt to address these issues and present an approach for designing feedback controllers that maximize the time-limited \emph{backward reachable set} (BRS), i.e. the set of points that reach a given target set at a specified finite time. Our approach is inspired by the method presented in \cite{Henrion12}, which describes a framework based on \emph{occupation measures} for computing the BRS for polynomial systems. In this paper, we extend this method to the control synthesis problem. Our contributions are three--fold. First, in Section \ref{sec:problemformulation}, we formulate the design of the feedback controller that generates the largest BRS as an infinite dimensional linear program (LP) over the space of nonnegative measures. Second, in Section \ref{sec:approxprob}, we construct a sequence of finite dimensional relaxations to our infinite dimensional LP in terms of semidefinite programs (SDPs). Finally, in Section \ref{sec:convergence}, we prove two convergence properties of our sequence of finite dimensional approximations: first that each solution to the sequence of SDPs is an \emph{outer approximation} to the largest possible BRS with asymptotically vanishing conservatism; and second, that there exists a subsequence of the SDP solutions that weakly converges to an optimizing solution of our original infinite dimensional LP.


The result of our analysis is a method capable of designing feedback controllers for nonlinear underactuated robotic systems in the presence of input saturation without resorting to linear analysis.
This is valuable for systems with degenerate linearizations, and can result in considerable improvements in performance for many practical robotic systems.  
Our method could also be used to augment existing feedback motion planning algorithms such as the LQR-Trees approach presented in \cite{Majumdar12a,Tedrake10}, which computes and sequences together BRSs in order to drive a desired set of initial conditions to some target set. 
Our approach could be substituted for the local, linear control synthesis employed by the aforementioned papers with the benefit of selecting control laws that maximize the size of the BRS in the presence of input saturations.
As a result, the number of trajectories required in a library in order to fill the space of possible initial conditions could be significantly reduced. In some cases, a \emph{single} nonlinear feedback controller could stabilize an entire set of initial conditions that previously required a library of locally-linear controllers.
We illustrate the performance of our approach in Section \ref{sec:examples} on five examples, whose source code we make available.

%
 
\subsection{Relationship to Lyapunov-Based Techniques}

Our approach is most comparable to those that use Lyapunov's criteria for stability in order to synthesize a controller that maximizes the region of attraction (ROA) of a particular target set. These criteria can be checked for polynomial systems by employing sums-of-squares (SOS) programming. However, the computation of the ROA and subsequent controller design are typically non-convex programs in this formulation \cite{Prajna04a}. 
The resulting optimization programs are bilinear in the decision variables and are generally solved by employing some form of bilinear alternation \cite{Jarvis-Wloszek05, Majumdar13}. Such methods are not guaranteed to converge to global optima (or necessarily even local optima) and require feasible initializations. 

The relationship between our approach and the Lyapunov-based approaches can be understood by examining the dual of our infinite dimensional LP, which is posed on the space of nonnegative continuous functions. This dual program certifies that a certain set cannot reach the target set within a pre-specified time for any valid control law. The complement of this set is an outer approximation of the BRS. This subtle change transforms the non-convex feedback control synthesis problem written in terms of Lyapunov's criteria into the convex synthesis problem which we present herein.  
The convexity of the control synthesis problem we present also has parallels to the convexity observed in \cite{Prajna04a} during the design of controllers to achieve global almost-everywhere asymptotic stability. However, this method is not easily extended to provide regional certificates, which are of greater practical utility in robotics since robotic systems are generally not globally stabilizable.


\section{Problem Formulation}
\label{sec:problemformulation}

In this section, we formalize our problem of interest, construct an infinite dimensional linear program (LP), and prove that the solution of this LP is equivalent to solving our problem of interest. We make substantial use of measure theory, and the unfamiliar reader may wish to consult \cite{Folland99} for an introduction.

\subsection{Notation}
Given an element $y \in \R^{n\times m}$, let $[y]_{ij}$ denote the $(i,j)$--th component of $y$.  We use the same convention for elements belonging to any multidimensional vector space.
 By $\NN$ we denote the non-negative integers, and $\NN_k^n$ refers to those $\alpha \in \NN^n$ with $|\alpha| = \sum_{i=1}^n [\alpha]_i \leq k$.
Let $\R[y]$ denote the ring of real polynomials in the variable $y$. 
For a compact set $K$, let $\M(K)$ denote the space of signed Radon measures supported on $K$. The elements of $\M(K)$ can be identified with linear functionals acting on the space of continuous functions $C(K)$, that is, as elements of the dual space $C(K)'$ \cite[Corollary 7.18]{Folland99}. The duality pairing of a measure $\mu \in \left(\M(K)\right)^p$ on a test function $v \in \left(C(K)\right)^p$ is:
\begin{equation}
\label{eq:duality_pairing}
\langle \mu, v \rangle = \sum_{i=1}^p\int_K [v]_i(z) d[\mu]_i(z).
\end{equation}

\subsection{Problem Statement}

Consider the control-affine system with feedback control
\begin{equation}
\label{eq:control_affine}
	\begin{aligned}
		\dot{x}(t) &= f\left(t,x(t)\right) + g\left(t,x(t)\right)u(t,x),
	\end{aligned}
\end{equation}
with state $x(t) \in \R^n$ and control action $u(t,x) \in \RR^m$, such that the components of the vector $f$ and the matrix $g$ are polynomials.
Our goal is to find a feedback controller, $u(t,x)$, that 
maximizes the BRS for a given target set while respecting the {\it input constraint}
\begin{equation}
  u(t,x) \in U = [a_1,b_1] \times \ldots\times  [a_m,b_m],
\end{equation}
where $\{a_j\}_{j=1}^m,\{b_j\}_{j=1}^m \subset \R$. Define the \emph{bounding set}, and \emph{target set} as:
\begin{equation}
	\begin{aligned}
		X &= \big\{x\in \R^n \mid h_{X_i}(x) \geq 0, \forall i = \{1,\ldots,n_X\} \big\}, \\
		X_T &= \big\{x\in \R^n \mid h_{T_i}(x) \geq 0, \forall i = \{1,\ldots,n_T\}\big\},
	\end{aligned}
\end{equation}
respectively, for given polynomials $h_{X_i},h_{T_i} \in \R[x]$.

Given a finite final time $T > 0$, let the BRS for a particular control policy $u \in L^1([0,T] \times X, U )$,  be defined as:
\begin{align}
	\label{eq:ROA}
		{\cal X}(u) = \Big\{ x_0 \in \R^n \mid &\dot{x}(t) = f\big(t,x(t)\big) + g\big(t,x(t)\big)u\big(t,x(t)\big) \nonumber \\ &\text{a.e.}~t \in [0,T],~ x(0) = x_0,~x(T) \in X_T, \nonumber \\& x(t) \in X~ \forall t \in [0,T] \Big\}.
\end{align}
${\cal X}(u)$ is the set of initial conditions for solutions\footnote{Solutions in this context are understood in the Carath\'eodory sense, that is, as absolutely continuous functions whose derivatives satisfy the right hand side of Equation \eqref{eq:control_affine} almost everywhere \cite[Chapter 10]{Aubin08}.} to Equation \eqref{eq:control_affine} that remain in the bounding set
and arrive in the target set at the final time when control law $u$ is applied. Our aim is to find a controller $u^*\in L^1([0,T] \times X, U )$, that maximizes the volume of the BRS:
\begin{equation}
\label{eq:control_synthesis_prob}
	\lambda( {\cal X}(u^*)) \geq \lambda( {\cal X}(u)), \quad \forall u \in L^1([0,T] \times X, U ),
\end{equation}
where $\lambda$ is the Lebesgue measure.
$u^*$ need not be unique. We denote the BRS corresponding to  $u^*$ by $\cal X^*$. To solve this problem, we make the following assumptions: 
\begin{assumption} \label{assume:ball}
	$X$ and $X_T$ are compact sets. 
\end{assumption}
\begin{remark} \label{rem:assumptions}
  Without loss of generality, we assume that $U = \{ u \in \R^m \mid -1 \leq u_j \leq 1 ~ \forall j \in \{1,\ldots,m\}\}$ (since $g$ can be arbitrarily shifted and scaled). 
  Assumption \ref{assume:ball} ensures the existence of a polynomial $h_{X_i}(x) = C_X - \left\|x \right\|_2^2$ for a large enough $C_X > 0$. 
\end{remark}

\subsection{Liouville's Equation}

We solve this problem by defining measures over $[0,T]\times X$ whose supports' model the evolution of \emph{families of trajectories}.
An initial condition and its relationship with respect to the terminal set can be understood via Equation \eqref{eq:control_affine}, but the relationship between a family of trajectories and the terminal set must be understood through a different lens.
First, define the linear operator $\Lf:C^1\big([0,T] \times X \big) \to C\big( [0,T] \times X \big)$ on a test function $v$ as:
\begin{equation}
\Lf v = \frac{\partial v}{\partial t} + \sum_{i = 1}^n \frac{\partial v}{\partial x_i} [f]_i(t,x),
\end{equation}
and its adjoint operator $\Lf': C\big( [0,T] \times X \big)' \to C^1\big([0,T] \times X \big)'$ by the adjoint relation:
\begin{equation}
\label{eq:adjoint_Lf}
\langle \Lf'\mu, v \rangle = \langle \mu, \Lf v \rangle = \int_{[0,T] \times X} \Lf v(t,x) d\mu(t,x)
\end{equation}
for all $\mu \in \M\big([0,T] \times X\big)$ and $v \in C^1\big([0,T] \times X \big)$. Define the linear operator $\Lg:C^1\big([0,T] \times X \big) \to C\big( [0,T]\times X\big)^m$ as:
\begin{equation}
[ \Lg v ]_j = \sum_{i = 1}^n \frac{\partial v}{\partial x_i} [g]_{ij}(t,x),
\end{equation}
for each $j \in \{1,\ldots,m\}$ and define its adjoint operator $\Lg': \left(C\big( [0,T] \times X \big)^m\right)' \to C^1\big([0,T] \times X \big)'$ according to its adjoint relation as in Equation \eqref{eq:adjoint_Lf}. Note that $\Lf v(t,x) + (\Lg v(t,x))u(t,x)$ is the time-derivative $\dot{v}$ of a function $v$.

\begin{figure}
	\centering
	\includegraphics[width=0.85\columnwidth]{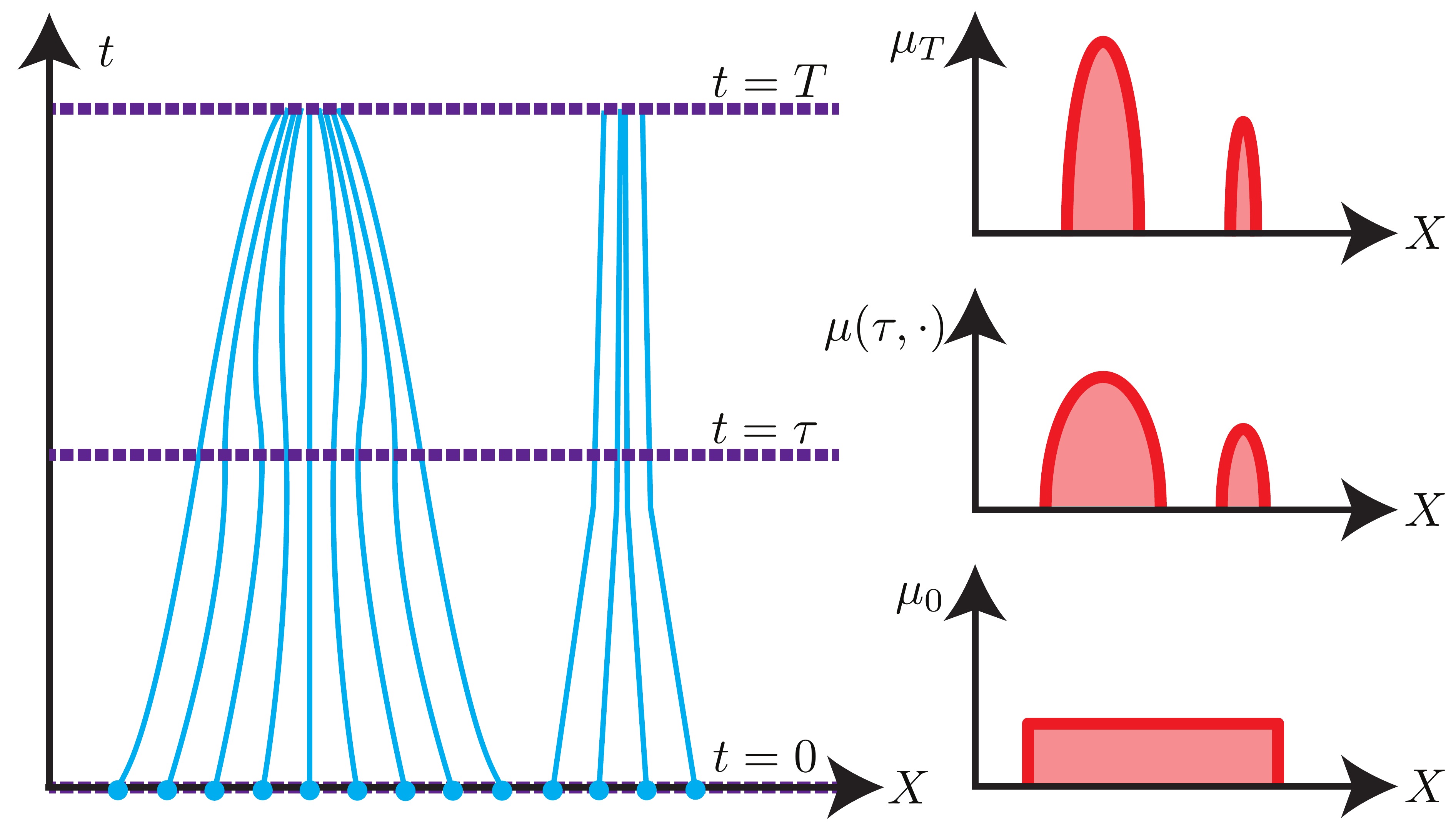}
	\vspace*{-5pt}
	\caption{An illustration (left) of trajectories (blue) transforming according to Equation \eqref{eq:control_affine} and their corresponding occupation measures (red) at times $0, \tau,$ and $T$ (purple) transforming according to Equation \eqref{eq:liouville}. \label{fig:occmeasures}}
	\vspace*{-22pt}
\end{figure}


Given a test function, $v \in C^1\left([0,T] \times X \right)$, and an initial condition, $x(0)\in X$, it follows that:
\begin{equation}
	\label{eq:test_func_evolve}
	v(T,x(T))  = v(0,x(0)) + \int_0^T \dot{v}\left(t,x(t|x_0)\right)dt.
\end{equation}
The traditional approach to designing controllers that stabilize the system imposes Lyapunov conditions on the test functions. However, simultaneously searching for a controller and Lyapunov function results in a nonconvex optimization problem \cite{Prajna04a}. Instead we examine conditions on the space of measures--the dual to the space of functions--in order to arrive at a convex formulation.

For a fixed control policy $u \in L^1\left( [0,T]\times X, U\right)$ and an initial condition $x_0 \in \R^n$, let  $x(\cdot | x_0) :[0,T] \to X$ be a solution to Equation \eqref{eq:control_affine}. Define the \emph{occupation measure} as:
\begin{equation}
	\mu(A\times B | x_0 ) = \int_0^T I_{A \times B}\left(t,x(t | x_0)) \right) dt,
\end{equation}
for all subsets $A \times B$ in the Borel $\sigma$-algebra of $[0,T] \times X$, where $I_{A\times B}(\cdot)$ denotes the indicator function on a set $A\times B$. This computes the amount of time the graph of the solution, $(t,x(t|x_0))$, spends in $A \times B$. Equation \eqref{eq:test_func_evolve} then becomes:
\begin{equation}
	\label{eq:test_func_evolve_occ_measure}
	\begin{aligned}
		v(T,x(T)) = v(0,&x(0)) + \int_{[0,T]\times X} \Big( \Lf v(t,x) + \\& + \Lg v(t,x)u(t,x) \Big) d\mu(t,x | x_0).
	\end{aligned}
\end{equation}
When the initial state is not a single point, but is a distribution modeled by an \emph{initial measure}, $\mu_0 \in \M(X)$, we define the \emph{average occupation measure}, $\mu \in \M\left([0,T] \times X\right)$ by:
\begin{equation}
	\mu(A\times B) = \int_X \mu(A \times B | x_0 ) d\mu_0(x_0),
\end{equation} 
and the \emph{final measure}, $\mu_T \in \M\left(X_T \right)$ by:
\begin{equation}
	\mu_T(B) = \int_X I_B( x(T|x_0) ) d\mu_0(x_0).
\end{equation}
Integrating with respect to $\mu_0$ and introducing the initial, average occupation, and final measures, Equation \eqref{eq:test_func_evolve_occ_measure} becomes:
\begin{multline}
	\label{eq:deriv_liouville}
	\int\limits_{X_T} v(T,x) d\mu_T(x) = \int\limits_X v(0,x) d\mu_0(x) + \\ + \int_{[0,T]\times X} \Big( \Lf v(t,x) + \Lg v(t,x)u(t,x) \Big) d\mu(t,x).
\end{multline}
It is useful to think of the measures $\mu_0$, $\mu$ and $\mu_T$ as \emph{unnormalized} probability distributions. The support of $\mu_0$ models the set of initial conditions, the support of $\mu$ models the flow of trajectories, and the support of $\mu_T$ models the set of states at time $T$.

Next, we subsume $u(t,x)$ into a \emph{signed} measure $\sigma^+ - \sigma^-$ defined by \emph{nonnegative} measures\footnote{Note that we can always decompose a signed measure into unsigned measures as a result of the Jordan Decomposition Theorem \cite[Theorem 3.4]{Folland99}.}  $\sigma^+,\sigma^- \in \left(\M\left([0,T]\times X \right)\right)^m$ such that:
\begin{equation}
	\label{eq:measure_control_synthesis}
	 \int_{A\times B}\hspace*{-.3cm} u_j(t,x) d\mu(t,x) = \int_{A\times B}\hspace*{-.35cm}d[\sigma^+]_j(t,x) - \hspace*{-.05cm}\int_{A\times B}\hspace*{-.35cm}d[\sigma^-]_j(t,x) 
\end{equation}
for all subsets $A\times B$ in the Borel $\sigma$-algebra of $[0,T] \times X$ and for each $j \in \{1,\ldots,m\}$. This key step allows us to pose an infinite dimensional LP over measures without explicitly parameterizing a control law while allowing us to ``back out'' a control law using Equation \eqref{eq:measure_control_synthesis}. Equation \eqref{eq:deriv_liouville} becomes:
\begin{equation}
\label{eq:testfnc_liouville}
	\langle \mu_T,v(T,\cdot) \rangle = \langle \mu_0, v(0,\cdot) \rangle + \langle \mu, \Lf v \rangle + \langle \sigma^+ \hspace*{-.03cm}- \sigma^-, \Lg v \rangle
\end{equation}
for all test functions $v \in C^1([0,T]\times X)$. Notice that this substitution renders Equation \eqref{eq:testfnc_liouville} linear in its measure components. Let $\delta_t$ denote the Dirac measure at a point $t$ and let $\otimes$ denote the product of measures. Since Equation \eqref{eq:testfnc_liouville} must hold for all test functions, we obtain a linear operator equation:
\begin{equation}
\label{eq:liouville}
\Lf'\mu + \Lg'\sigma^+ - \Lg'\sigma^- = \delta_T \otimes \mu_T - \delta_0 \otimes \mu_0,
\end{equation}
called Liouville's Equation, which is a classical result in statistical physics that describes the evolution of a density of particles within a fluid \cite{Arnold89}. Figure \ref{fig:occmeasures} illustrates the evolution of densities according to Liouville's Equation. This equation is satisfied by families of admissible trajectories starting from the initial distribution $\mu_0$. The converse statement is true for control affine systems with a convex admissible control set, as we have assumed. We refer the reader to \cite[Appendix A]{Henrion12} for an extended discussion of Liouville's Equation.  

\subsection{BRS via an Infinite Dimensional LP}\label{subsec:ROA_LP}

The goal of this section is to use Liouville's Equation to formulate an infinite dimensional LP, $P$, that maximizes the size of the BRS, modeled by $\supp(\mu_0)$, for a given target set, modeled by $\supp(\mu_T)$, where $\supp(\mu)$ denotes the support of a measure $\mu$. Slack measures (denoted with ``hats'') are used to impose the constraints $\lambda \geq \mu_0$ and $\mu \geq [\sigma^+]_j + [\sigma^-]_j$ for each $j \in \{1,\dots,m\}$, 
where $\lambda$ is the Lebesgue measure. The former constraint ensures that the optimal value of $P$ is the Lebesgue measure of the largest achievable BRS (see Theorem \ref{thm:optP}). The latter constraint ensures that we are able to extract a bounded control law by applying Equation \eqref{eq:measure_control_synthesis} (see Theorem \ref{thm:lp control}). Define $P$ as:
\begin{flalign} 
			& \text{sup} & & \mu_0(X) && \hspace*{-1.75cm}(P) \nonumber \\
			& \text{s.t.} & & \Lf'\mu + \Lg'(\sigma^+-\sigma^-) =
                        \delta_T \otimes \mu_T - \delta_0 \otimes
                        \mu_0, \nonumber \\
			& & & [\sigma^+]_j + [\sigma^-]_j + [\hat{\sigma}]_j = \mu &&\hspace*{-1.75cm}\forall j \in \{1,\ldots,m\}, \nonumber \\
			& & & \mu_0 + \hat{\mu}_0 = \lambda,&& \nonumber \\
            & & & [\sigma^+]_j, [\sigma^-]_j, [\hat{\sigma}]_j \geq 0 &&\hspace*{-1.75cm}\forall j \in \{1,\ldots,m\}, \nonumber\\
			& & & \mu, \mu_0, \mu_T, \hat{\mu}_0 \geq 0, \nonumber
\end{flalign}
where the given data are $f,g,X,X_T$ and the supremum is taken
over a tuple of measures $(\sigma^+,\sigma^-,\hat
\sigma,\mu,\mu_0,\hat{\mu}_0,\mu_T) \in \left(\M\big([0,T] \times X\big)\right)^m\times \left(\M\big([0,T] \times X\big)\right)^m\times \left(\M\big([0,T] \times X\big)\right)^m\times\M\big([0,T] \times X\big) \times \M(X) \times \M(X) \times \M(X_T)$.  Given measures that achieve the supremum, the control law that maximizes the size of the BRS is then constructed by finding the $u\in L^1([0,T] \times X, U )$ whose components each satisfy Equation \eqref{eq:measure_control_synthesis} for all subsets in the Borel $\sigma$-algebra of $[0,T] \times X$. Before proving that this two-step procedure computes $u^*\in L^1([0,T] \times X, U )$ as in Equation \eqref{eq:control_synthesis_prob}, define the dual program to $P$ denoted $D$ as:
\begin{flalign} 
		& \text{inf} & & \int_X w(x) d\lambda(x) && (D) \nonumber \\
		& \text{s.t.} & & \Lf v +\Sigma_{i=1}^m [p]_i \leq 0, \nonumber\\
		& & & [p]_i \geq 0, \quad [p]_i \geq |[\Lg v]_i|  && \forall i = \{1,\dots,m\}, \nonumber \\
        & & & w \geq 0, && \nonumber \\
		& & & w(x) \geq v(0,x) + 1 &&\forall x \in X, \nonumber \\
		& & & v(T,x) \geq 0 &&\forall x \in X_T \nonumber
\end{flalign}
where the given data are $f,g,X,X_T$ and the infimum is taken over $(v,w,p) \in C^1\left([0,T] \times X \right) \times C(X) \times \left(C([0,T] \times X)\right)^m $. The dual allows us to obtain approximations of the BRS $\cal X^*$ (see Theorem \ref{thm:w sequence}).

\begin{theorem}
There is no duality gap between $P$ and $D$.
\end{theorem}
\begin{proof}
Due to space limitations, we omit the proof, which follows from \cite[Theorem 3.10]{Anderson87}.
\end{proof}
\begin{theorem}
	\label{thm:optP}
The optimal value of $P$ is equal to $\lambda(\cal X^*)$, the Lebesgue measure of the BRS of the controller defined by Equation \eqref{eq:measure_control_synthesis}. 
\end{theorem}
\begin{proof}
Since there is no duality gap between $P$ and $D$, it is sufficient to show that the optimal value of $D$ is equal to $\lambda(\cal X^*)$. We do this by demonstrating that $D$ is equivalent to the dual LP defined in Equation (15) in \cite{Henrion12}, whose optimal value is equal to $\lambda(\cal X^*)$ \cite[Theorem 1]{Henrion12}. Note that the constraints $w(x) \geq v(0,x) + 1$, $v(T,x) \geq 0$, and $w(x) \geq 0$ appear in both optimization problems. Since the objectives are also identical, it suffices to show that the first three constraints in $D$ are equivalent to the constraint $\Lf v(t,x) + (\Lg v(t,x)) u \leq 0 \ \forall (t,x,u) \in [0,T] \times X \times U$. Suppose that the former set of the three constraints holds. Given $u \in U$, note that $\Lf v + (\Lg v)u \leq \Lf v + \Sigma_{i=1}^m |[\Lg v]_i u_i|$. Hence, since $[p]_i \geq |[\Lg v]_i|$, $\Lf v + \Sigma_{i=1}^m [p]_i \leq 0$, and $|u_i| \leq 1$ (see Remark \ref{rem:assumptions}), we have the desired result.

To prove the converse, we illustrate the existence of $[p]_i(t,x) \geq 0$ that satisfies the three constraints appearing in $D$. Let $[p]_i(t,x) = |\Lg v(t,x)]_i|$, which is a non-negative continuous function. Clearly, $p_i \geq [\Lg v]_i$ and $[p]_i \geq -[\Lg v]_i$. To finish the proof, note:
\begin{align*}
\Lf v(t,x) + \Sigma_{i=1}^m [p]_i(t,x) 
 & = \underset{u \in U}{\text{sup}} \Lf v(t,x) + \Lg v(t,x)) u \leq 0 
\end{align*}
\end{proof}

The solution to $P$ can be used in order to construct the control law that maximizes the BRS:
\begin{theorem} \label{thm:lp control}
	There exists a control law, $\tilde{u}\in L^1([0,T] \times X, U )$, that satisfies Equation \eqref{eq:measure_control_synthesis} when substituting in the vector of measures that achieves the supremum of $P$, $(\sigma^{+*},\sigma^{-*},\hat{\sigma}^{*},\mu^*,\mu^*_0,\hat{\mu}^*_0,\mu^*_T)$, and is the control law that maximizes the size of the BRS, i.e. $\lambda({\cal X}(\tilde{u})) \geq \lambda( {\cal X}(u)), \forall u \in L^1([0,T] \times X, U )$. Moreover, any two control laws constructed by applying Equation \eqref{eq:measure_control_synthesis} to the vector of measures that achieves the supremum of $P$ are equal $\mu^*$-almost everywhere. 
\end{theorem}
\begin{proof}
Note that $[\sigma^{+*}]_j,[\sigma^{-*}]_j,$ and $\mu^*$ are $\sigma$-finite for all $j \in \{1,\ldots,m\}$ since they are Radon measures defined over a compact set. Define $[\sigma^*]_j = [\sigma^{+*}]_j - [\sigma^{-*}]_j$ for each $j \in \{1,\ldots,m\}$ and notice that each $[\sigma^*]_j$ is also $\sigma$-finite. Since $[\sigma^{+*}]_j + [\sigma^{-*}]_j + [\hat{\sigma}^*]_j = \mu^*$ and $[\sigma^{+*}]_j, [\sigma^{-*}]_j, [\hat{\sigma}^*]_j \geq 0$, $\sigma^*$ is absolutely continuous with respect to $\mu^*$. Therefore as a result of the Radon--Nikodym Theorem \cite[Theorem 3.8]{Folland99}, there exists a $\tilde{u}\in L^1([0,T] \times X, U )$, which is unique $\mu^*$-almost everywhere, that satisfies Equation \eqref{eq:measure_control_synthesis} when plugging in the vector of measures that achieves the supremum of $P$. To see that $\lambda({\cal X}(\tilde{u})) \geq \lambda( {\cal X}(u)), \forall u \in L^1([0,T] \times X, U )$, notice that by construction  $\mu^*_T,\mu^*_0,\mu^*$, and $\tilde{u}$ satisfy Equation \eqref{eq:deriv_liouville} for all test functions $v \in C^1([0,T]\times X)$. Since $\mu^*_0$ describes the maximum BRS and Equation \eqref{eq:deriv_liouville} describes all admissible trajectories, we have our result. 
\end{proof}

Next, we note that the $w$-component to a feasible point of $D$ is an outer approximation to ${\cal X}^*$. 
This follows from our proof of Theorem \ref{thm:optP} and Lemma 2 and Theorem 3 in \cite{Henrion12}.

\begin{theorem} \label{thm:w sequence}
$\cal X^*$ is a subset of $\{x \ | \ w(x) \geq 1\}$, for any feasible $w$ of the $D$. Furthermore, there is a sequence of feasible solutions to $D$ such that the $w$-component converges from above to $I_{\cal X^*}$ in the $L^1$ norm and almost uniformly. 
\end{theorem}
\section{Numerical Implementation}
\label{sec:implementation}

The infinite dimensional problems $P$ and $D$ are not directly amenable to computation. However, a sequence of finite dimensional approximations in terms of semidefinite programs (SDPs) can be obtained by characterizing measures in $P$ by their \emph{moments}, and restricting the space of functions in $D$ to polynomials. The solutions to each of the SDPs in this sequence can be used to construct controllers and outer approximations that converge to the solution of the infinite dimensional LP. A comprehensive introduction to such \emph{moment relaxations} can be found in \cite{Lasserre10}.

Measures on the set $[0,T] \times X$ are completely determined by their action (via integration) on a dense subset of the space $C^1([0,T] \times X)$ \cite{Folland99}. Since $[0,T] \times X$ is compact, the Stone-Weierstrass Theorem \cite[Theorem 4.45]{Folland99} allows us to choose the set of polynomials as this dense subset. Every polynomial on $\RR^n$, say $p \in \RR[x]$ with $x = (x_1,\ldots,x_n)$, can be expanded in the monomial basis via
$$p(x) = \sum_{\alpha\in\NN^n} p_\alpha x^\alpha,$$
where $\alpha = (\alpha_1,\ldots,\alpha_n)$ ranges over vectors of non-negative integers, $x^\alpha = x_1^{\alpha_1} \ldots x_n^{\alpha_n}$, and $\vect(p) = (p_\alpha)_{\alpha \in \NN^n} $ is the vector of coefficients of $p$.
By definition, the $p_\alpha$ are real and only finitely many are non-zero.
We define $\RR_k[x]$ to be those polynomials such that $p_\alpha$ is non-zero only for $\alpha \in \NN_k^n$. The degree of a polynomial, $\deg(p)$, is the smallest $k$ such that $p \in \RR_k[x]$.


The moments of a measure $\mu$ defined over a real $n$-dimensional space are given by:
\begin{equation}
y^\alpha_{\mu} = \int x^\alpha d\mu(x).
\end{equation} 
Integration of a polynomial with respect to a measure $\nu$
can be expressed as a linear functional of its coefficients:
\begin{equation}
\langle \mu,p \rangle = \int p(x) d\mu(x)  = \sum_{\alpha \in \N^n} p_\alpha y^\alpha_{\mu} = \vect(p)^T y_{\mu}.
\end{equation}
Integrating the square of a polynomial $p \in \RR_k[x]$, we obtain:
\begin{equation}
\int p(x)^2 d\mu(x) = \vect(p)^T M_k({y_\mu}) \vect(p),
\end{equation}
where $M_k(y_\mu)$ is the \emph{truncated moment matrix} defined by
\begin{equation}
  [M_k(y_\mu)]_{(\alpha,\beta)} = y^{\alpha+\beta}_\mu
\end{equation}
for $\alpha, \beta \in \NN_k^n$.  Note that for any positive measure $\mu$, the matrix $M_k(y_\mu)$ must be positive semidefinite.
Similarly, given  $h \in \R[x]$ with $(h_\gamma)_{\gamma \in \NN^n} = \vect(h)$ one has
\begin{equation}
\int p(x)^2 h(x)d\mu(x) = \vect(p)^T M_k(h,{y_\mu}) \vect(p),
\end{equation}
where $M_k(h,y)$ is a  \emph{localizing matrix} defined by
\begin{equation}
[M_k(h,y_\mu)]_{(\alpha,\beta)} = \sum_{\gamma \in \NN^n} h_\gamma y^{\alpha+\beta}_\mu
\end{equation}
for all $\alpha,\beta \in \NN_k^n$.  
The localizing and moment matrices are symmetric and linear in the moments $y$.

\subsection{Approximating Problems}
\label{sec:approxprob}

Finite dimensional SDPs approximating $P$ can be obtained by replacing constraints on measures with \emph{constraints on moments}. All of the equality constraints of $P$ can be expressed as an infinite dimensional linear system of equations which the moments of the measures appearing in $P$ must satisfy. This linear system is obtained by restricting to polynomial test functions (which we note are sufficient given our discussion above): $v(t,x) = t^\alpha x^\beta$, $[p]_j(t,x) = t^\alpha x^\beta$, and $w(x) = x^\beta$, $\forall \alpha \in \mathbb{N}, \beta \in \mathbb{N}^n$. For example, the equality constraint corresponding to Liouville's Equation is obtained by examining:
\begin{align}
	0 &= \hspace*{-.4cm}\int\limits_{[0,T] \times X}\hspace*{-.3cm} \Lf (t^\alpha x^\beta) d\mu(t,x) + \hspace*{-.4cm}\int\limits_{[0,T] \times X}\hspace*{-.3cm} \Lg (t^\alpha x^\beta) d[\sigma^+]_j(t,x)  \nonumber \\ - \hspace*{-.4cm}\int\limits_{[0,T] \times X}\hspace*{-.3cm}&\Lg (t^\alpha x^\beta) d[\sigma^-]_j(t,x) - \hspace*{-.1cm}\int\limits_{X_T}T^\alpha x^\beta d\mu_T(x) 
	 + \hspace*{-.1cm}\int \limits_X x^\beta d\mu_0(x). \nonumber
\end{align}
A \emph{finite dimensional} linear system is obtained by truncating the degree of the polynomial test functions to $2k$. Let $\Gamma = \{\sigma^+,\sigma^-,\hat{\sigma},\mu,\mu_0,\hat{\mu}_0,\mu_T\}$, then let $\mathbf{y}_k = (y_{k,\gamma}) \subset \R$ be a vector of sequences of moments truncated to degree $2k$ for each $\gamma \in \Gamma$. The finite dimensional linear system is then represented by the  linear system:
\begin{equation}
A_k(\mathbf{y}_k) = b_k.
\end{equation}
Constraints on the support of the measures also need to be imposed (see \cite{Lasserre10} for details). 
Let the $k$-th relaxed SDP representation of $P$, denoted $P_k$, be defined as:
\begin{flalign} \label{eq:primal sdp}
		 & \text{sup} & &  y_{k, \mu_0}^0 && (P_k) \nonumber\\
		 & \text{s.t.} & &  A_k(\mathbf{y}_k) = b_k, && \nonumber\\
         & & & M_k(y_{k,\gamma}) \succeq 0 && \forall \gamma \in \Gamma, \nonumber\\
		 & & & M_{k_{X_i}}(h_{X_i},y_{k,\gamma}) \succeq 0 && \forall (i,\gamma) \in \{1, \ldots, n_X\} \times \Gamma\backslash \mu_T, \nonumber\\
	     & & & M_{k_{T_i}}(h_{T_i},y_{k,\mu_T}) \succeq 0 && \forall i \in \{1, \ldots, n_T\}, \nonumber\\
		 & & & M_{k-1}(h_{\tau},y_{k,\gamma}) \succeq 0 && \forall \gamma \in \Gamma \backslash \{\mu_0,\mu_T,\hat{\mu}_0 \}, \nonumber
\end{flalign}
where the given data are $f,g,X,X_T$ and the supremum is taken over the sequence of moments, $\mathbf{y}_k = (y_{k,\gamma})$, $h_{\tau} = t(T-t)$, $k_{X_i} = k - \lceil \text{deg}(h_{X_i})/2 \rceil$, $k_{T_i} = k - \lceil \text{deg}(h_{T_i})/2 \rceil$, and $\succeq 0$ denotes positive semi-definiteness. For each $k \in \N$, let $\mathbf{y}^*_k$ denote the optimizer of $P_k$, with components $y^*_{k,\gamma}$ where $\gamma \in \Gamma$ and let $p^*_k$ denote the supremum of $P_k$.


The dual of $P_k$ is a sums-of-squares (SOS) program denoted $D_k$ for each $k \in \N$, which is obtained by first restricting the optimization space in the $D$ to the polynomial functions with degree truncated to $2k$ and by then replacing the non-negativity constraint $D$ with a \emph{sums-of-squares} constraint  \cite{Parrilo00}. Define $Q_{2k}(h_{X_1},\ldots,h_{X_{n_X}}) \subset \mathbb{R}_{2k}[x]$ to be the set of polynomials $q \in \mathbb{R}_{2k}[x]$ (i.e. of total degree less than $2k$) expressible as
\begin{equation}
  q = s_0 + \sum_{i=1}^{n_X} s_i h_{X_i},
\end{equation}
for some polynomials $\{s_i\}_{i=0}^{n_X} \subset \RR_{2k}[x]$ that are sums of squares of other polynomials.  Every such polynomial is clearly non-negative on $X$. Define $Q_{2k} (h_{\tau},h_{X_1},\ldots,h_{X_{n_X}}) \subset \mathbb{R}_{2k}[t,x]$ and $Q_{2k} (h_{T_1},\ldots,h_{T_{n_T}}) \subset \mathbb{R}_{2k}[x]$, similarly. Employing this notation, the $k$-th relaxed SDP representation of $D$, denoted $D_k$, is defined as:
\begin{flalign}
		&\text{inf} & &  l^T \textrm{vec}(w) && (D_k) \nonumber\\
                &\text{s.t.} && -\Lf v - {\bf 1}^Tp \in  Q_{2k} (h_{\tau},h_{X_1},\ldots,h_{X_{n_X}}), &\nonumber\\
               & & & p - (\Lg v)^T  \in (Q_{2k} (h_{\tau},h_{X_1},\ldots,h_{X_{n_X}}))^m,&\nonumber\\
               & & & p + (\Lg v)^T  \in (Q_{2k} (h_{\tau},h_{X_1},\ldots,h_{X_{n_X}}))^m,&\nonumber\\
               & & & w \in  Q_{2k} (h_{X_1},\ldots,h_{X_{n_X}}),&\nonumber\\
               & & &  w-v(0,\cdot)- 1  \in  Q_{2k} (h_{X_1},\ldots,h_{X_{n_X}}),& \nonumber\\
			   & & & v(T,\cdot)  \in  Q_{2k} (h_{T_1},\ldots,h_{T_{n_T}}),&\nonumber
\end{flalign}
where the given data are $f,g,X,X_T$, the infimum is taken over the vector of polynomials $(v,w,p) \in \mathbb{R}_{2k}[t,x] \times \mathbb{R}_{2k}[x] \times (\mathbb{R}_{2k}[t,x])^m$, and $l$ is a vector of moments associated with the Lebesgue measure (i.e. $\int_X w\ d\lambda = l^T\textrm{vec}(w)$ for all $w \in \mathbb{R}_{2k}[x]$). For each $k \in \N$, let $d^*_k$ denote the infimum of $D_k$.

\begin{theorem}
	For each $k \in \N$, there is no duality gap between $P_k$ and $D_k$.
\end{theorem}
\begin{proof}
This follows from standard results from the theory of SDP duality and we do not include the full proof here. The proof involves noting that the moment vectors in SDP, $P_k$, are necessarily bounded because of the constraint $\mu_0 + \hat{\mu}_0 = \lambda$, and then arguing that the feasible set of the SDP, $D_k$, has an interior point. The existence of an interior point is sufficient to establish zero duality gap \cite[Theorem 5]{Trnovska05}. 
\end{proof}


Next, we construct a technique to extract a polynomial control law from the solution $\mathbf{y}_k$ of $P_k$. 
Given moment sequences truncated to degree $2k$, one can choose an approximate control law $u_k$  with components $[u_k]_j \in \RR_k[t,x]$ so that the truncated analogue of Equation \eqref{eq:measure_control_synthesis} is satisfied. That is, by requiring:
\begin{equation}
 \int\limits_{[0,T] \times X} \hspace*{-.3cm} t^{\alpha_0} x^\alpha [u_k]_j(t,x) \; d\mu(t,x) = \int\limits_{[0,T] \times X} \hspace*{-.3cm} t^{\alpha_0} x^\alpha d[\sigma^+ - \sigma^{-}]_j,
\end{equation}
for $(\alpha_0,\alpha)$ satisfying  $\sum_{i=0}^n \alpha_i \leq k$. When constructing a polynomial control law from the solution of $P_k$, these linear equations written with respect to the coefficients of $[u_k]_j$ are expressible in terms of $y^*_{k,\sigma^+}, y^*_{k,\sigma^-},$ and $y^*_{k,\mu}$. Direct calculation shows the linear system of equations is:
\begin{equation} \label{eq:controllers}
M_k(y_{k,\mu}^*) \textrm{vec}([u_k]_j) = y^*_{k,[\sigma^+]_j}-y^*_{k,[\sigma^-]_j}.
\end{equation}

\subsection{Convergence of Approximating Problems}
\label{sec:convergence}

Next, we prove the convergence properties of $P_k$ and $D_k$ and the corresponding controllers. We begin by proving that the polynomial $w$ approximates the indicator function of the set $\cal X^*$. As we increase $k$, this approximation gets tighter. The following theorem makes this statement precise. 
\begin{theorem} \label{thm:w convergence}
For each $k \in \N$, let $w_k \in \R_{2k}[x]$ denote the w-component of the solution to $D_k$, and let $\bar{w}_k(x) = \textrm{min}_{i \leq k} w_i(x)$. Then, $w_k$ converges from above to $I_{\cal X^*}$ in the $L^1$ norm, and $\bar{w}_k(x)$ converges from above to $I_{\cal X^*}$ in the $L^1$ norm and almost uniformly.
\end{theorem}
\begin{proof}
From Theorem \ref{thm:w sequence}, for every $\epsilon > 0$, there exists a feasible tuple of functions $(v,w,p) \in C^1\left([0,T] \times X \right) \times C(X) \times \left(C([0,T] \times X)\right)^m$ such that $w \geq I_{\cal X^*}$ and $\int_X (w - I_{\cal X^*}) d\lambda < \epsilon$. Let $\tilde{v}(t,x) := v(t,x) - 3\epsilon T + 3(T+1)\epsilon$, $\tilde{w}(x) := w(x) + 3(T+3)\epsilon$ and $[\tilde{p}]_i(t,x) = [p]_i(t,x) + (2\epsilon)/m, \forall i = \{1,\dots,m\}$. Then, $\Lf \tilde{v} = \Lf v - 3\epsilon$, $\tilde{v}(t,x) = v(T,x) + 3\epsilon$, $\tilde{w}(x) - \tilde{v}(0,x) = 1 + 6\epsilon$, and $\Lg \tilde{v} = \Lg v$. Since the sets $X$ and $[0,T] \times X$ are compact, and by a generalization of the Stone-Weierstrass theorem that allows for the simultaneous uniform approximation of a function and its derivatives by a polynomial \cite[pp. 65-66]{Hirsch76}, we are guaranteed the existence of polynomials $\hat{v}, \hat{w}, [\hat{p}]_i$ such that $\| \hat{v} - \tilde{v} \|_{\infty} < \epsilon$, $\| \Lf \hat{v} - \Lf \tilde{v} \|_{\infty} < \epsilon$, $\| \Lg \hat{v} - \Lg \tilde{v} \|_{\infty} < \epsilon/m$, $\| \hat{w} - \tilde{w} \|_{\infty} < \epsilon$ and  $\| [\hat{p}]_i - [\tilde{p}]_i \|_{\infty} < \epsilon/m$. It is easily verified that these polynomials \emph{strictly} satisfy the constraints of $D_k$. Hence, by Putinar's Positivstellensatz \cite{Lasserre10} and Remark \ref{rem:assumptions}, we are guaranteed that these polynomials are feasible for $D_k$ for high enough degree of multiplier polynomials. We further note that $\hat{w} \geq w$. Then, $\int_X |\tilde{w} - \hat{w}| d\lambda \leq \epsilon \lambda(X)$, and thus $\int_X (\hat{w} - w) d\lambda \leq \epsilon \lambda(X) (3T + 10)$. Hence, since $w \geq I_{\cal X^*}$ and $\int_X (w - I_{\cal X^*}) d\lambda < \epsilon$ by assumption, it follows that $\int_X (\hat{w} - I_{\cal X^*}) d\lambda < \epsilon(1 + \lambda(X)(3T + 10))$ and $\hat{w} \geq I_{\cal X^*}$. This concludes the first part of the proof since $\epsilon$ was arbitrarily chosen.

The convergence of $w_k$ to $I_{\cal X^*}$ in $L^1$ norm implies the existence of a subsequence $w_{k_i}$ that converges almost uniformly to $I_{\cal X^*}$ ~\cite[Theorems 2.5.2, 2.5.3]{Ash72}. Since $\bar{w}_k(x) \leq \textrm{min}\{w_{k_i} : k_i \leq k\}$, this is sufficient to establish the second claim.
\end{proof}

\begin{corollary} \label{thm:convergence of dual sdp}
 $\{d^*_k\}_{k=1}^{\infty}$ and $\{p^*_k\}_{k=1}^{\infty}$ converge monotonically from above to the optimal value of $D$ and $P$. 
\end{corollary}
\begin{proof}
This is a direct consequence of Theorem \ref{thm:convergence of dual sdp}.
\end{proof}

Next, we prove that the $1$-superlevel set of the polynomial $w$ converges in Lebesgue measure to the largest achievable BRS $\cal X^*$.

\begin{theorem}
For each $k \in \N$, let $w_k \in \R_{2k}[x]$ denote the w-component of the solution to $D_k$, and let ${\cal X}_k := \{x \in \R^n \ : \ w_k(x) \geq 1\}$. Then, $\lim_{k \to \infty} \lambda({\cal X}_k \backslash {\cal X}^*) = 0$.
\end{theorem}
\begin{proof}
Using Theorem \ref{thm:w sequence} we see $w_k \geq I_{{\cal X}_k} \geq I_{{\cal X}^*}$. From Theorem \ref{thm:w convergence}, we have $w_k \to I_{{\cal X}^*}$ in $L^1$ norm on $X$. Hence: 
\begin{equation}
\begin{aligned}
\lambda({\cal X}^*) &= \lim_{k \to \infty} \int_X w_k d\lambda 
& \geq \lim_{k \to \infty} \int_X I_{{\cal X}_k} d\lambda = \lim_{k \to \infty} \lambda({\cal X}_k). \nonumber
\end{aligned}
\end{equation}
But since ${\cal X}^* \subset {\cal X}_k$ for all $k$, we must have $\lim_{k \to \infty} \lambda({\cal X}_k) = \lambda({\cal X}^*)$ and thus $\lim_{k \to \infty} \lambda({\cal X}_k \backslash {\cal X}^*) = 0$.
\end{proof}

Finally, we prove a convergence result for the sequence of controllers generated by \eqref{eq:controllers}. For each $k \in \N$, let $u^*_k$ denote the controller constructed by Equation \eqref{eq:controllers} using the optimizers $\mathbf{y}_k$ of $P_k$. Let $y_{k,\mu}^*$ be the optimizing moment sequence corresponding to $\mu$. 
\begin{theorem} \label{thm:controller convergence}
	Let $\{\mu^*_k\}_{k = 1}^\infty$ be any sequence of measures such that the truncated moments of $\mu^*_k$ match $y_{k,\mu}^*$. Then, there exists an optimizing vector of measures $(\sigma^{+*},\sigma^{-*},\hat{\sigma}^*,\mu^*,\mu^*_0,\hat{\mu}^*_0,\mu^*_T)$ for $P$, a $u^* \in L^1([0,T]\times X)$ generated using $\sigma^{+*},\sigma^{-*},$ and $\mu^*$ according to Equation \eqref{eq:measure_control_synthesis}, and a subsequence $\{k_i\}_{i=1}^{\infty} \subset \N$ such that:
	\begin{equation} \label{eq:integral}
		\int_{[0,T] \times X}\hspace*{-1cm} v(t,x)\big([u^*_{k_i}]_j(t,x) d\mu^*_{k_i}(t,x) - [u^*]_j(t,x) d\mu^*(t,x)\big) \hspace*{-.1cm} \xrightarrow{i\to\infty} \hspace*{-.1cm} 0, 
		\end{equation}
for all $v \in C^1([0,T]\times X)$, and each $j \in \{1,\dots,m\}$.

\end{theorem}
\begin{proof}
We provide only a sketch of the proof due to space restrictions. First, note that the set of test functions $v$ can be restricted to polynomials since the set of polynomials is dense in $C^1([0,T]\times X)$.  Further, the construction of $u^*_{k_i}$ from Equation \eqref{eq:controllers} ensures that \eqref{eq:integral} holds for $v$ up to degree $k_i$. The rest follows directly from the proof of Theorem 4.7 in \cite{Lasserre10}. 
\end{proof}

\section{Examples} 
\label{sec:examples}
This section provides a series of numerical experiments on example
systems of increasing complexity.  SDPs were prepared using a custom
software toolbox and the modeling tool YALMIP \cite{Lofberg04}.  For
simulations, control laws are taken to be the saturation of
the polynomial law derived by the proposed method.  The programs, whose source code is available for download\footnote{\url{https://groups.csail.mit.edu/locomotion/software.html}}, are
solved using SeDuMi 1.3 \cite{Sturm99}, for the first three and last examples, and the SDPT3 solver \cite{Tutuncu03}, for the fourth example, on a machine with 8 Intel Xeon
processors with a clock speed of 3.1 GHz and 32 GB RAM.  

Additionally, for several of the examples we examine a different objective wherein we look to drive initial conditions starting in $X$ to $X_T$ at
any time $t \in [0,T]$ (commonly referred to as a ``free final time''
problem).  The analogous BRS approximation problem is addressed in
\cite{Henrion12} and the control synthesis problem follows using our approach in a straightforward manner. We make clear when we employ this different objective while describing each of our examples.

\subsection{Double Integrator}
The double integrator is a two state, single input 
system given by $\dot x_1 = x_2, \ \dot x_2 = u$, 
with $u$ restricted to the interval $U = [-1,1]$.  Setting the target
set to the origin, $X_T = \{0\}$, the optimal BRS $\mathcal{X}^*$ can be computed analytically based on a minimum time ``bang-bang controller'' \cite[pp. 136]{Bertsekas05a}. Note that this is a challenging system for grid based optimal control methods, since they require high resolution near the switching surface of the ``bang-bang'' control law.
 
We take the bounding
set to be  $X = \{x \ | \ \|x\|^2 \leq 1.6^2 \}$. 
Figure~\ref{fig:double_integrator_w} compares the outer approximations
of $\mathcal{X}^*$ for $k = 2,3,4$.  The
quality of the approximations increases quickly. Figure~\ref{fig:double_integrator_w} also evaluates the
performance of the control laws $u_k$ by plotting the terminal states $x(T)$ for controlled solutions starting in $\mathcal{X}^*$.
Even for $k = 3$, reasonable performance is achieved. The running times for $k = 2,3,4$ are $0.3 s$, $0.7 s$, and $4.2 s$, respectively. 

\begin{figure}
  \centering 
\includegraphics[width=0.8\columnwidth]{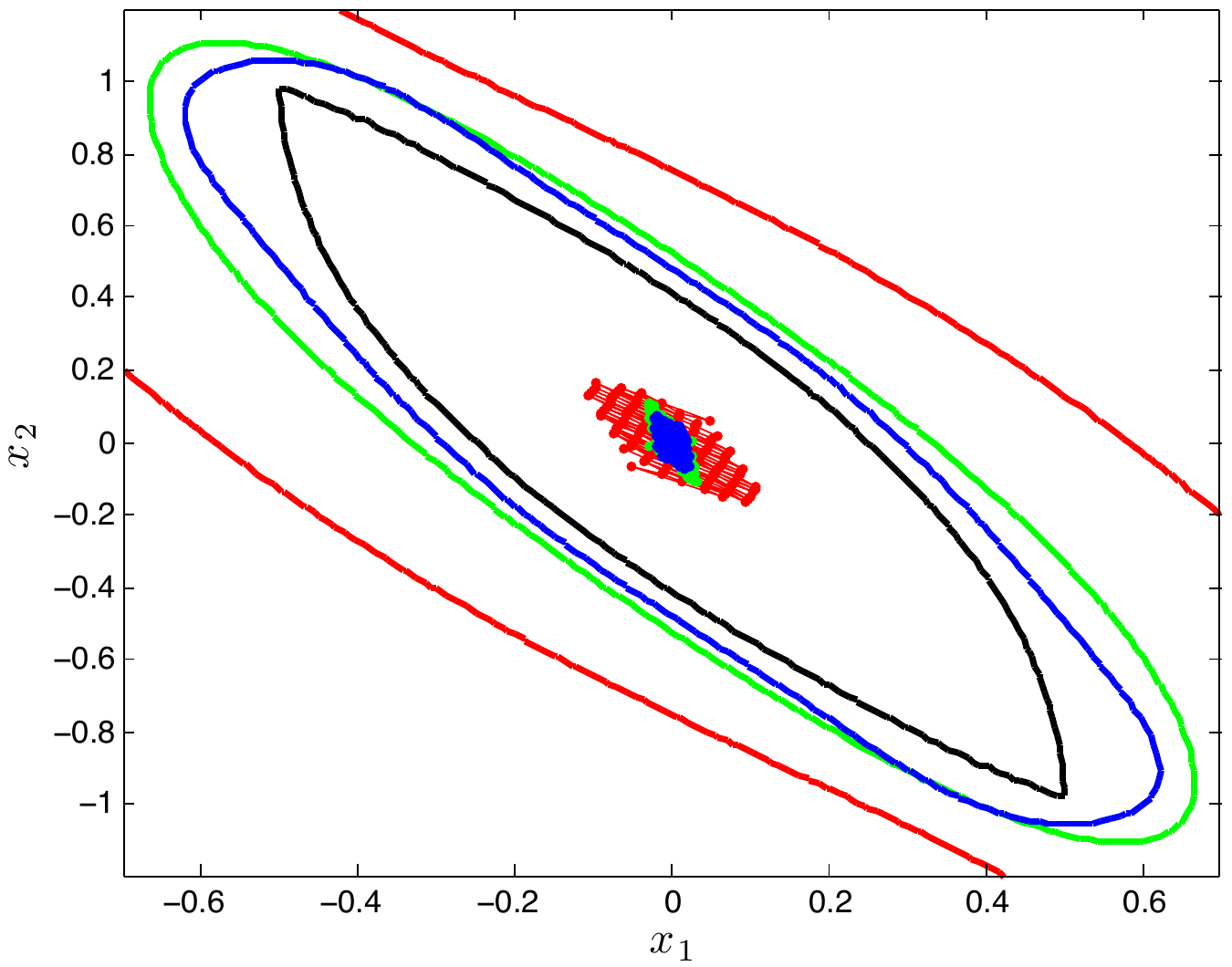}
\vspace*{-7pt}
 \caption{ An illustration of the convergence of outer approximations and the
  performance of controllers designed using our approach for increasing truncation degree, $k$, for the double integrator.  Solid
  lines indicate the outer approximations, defined by $w_k = 1$,
  for $k = 2$ (red), $k = 3$ (green), $k = 4$ (blue), and the boundary
  of the true BRS (solid black). Points indicate terminal states
  ($x(T)$) of controlled solutions with $x(0)$ inside the BRS using
  our generated feedback control laws $u_k$ (colors match the outer approximations).
	\label{fig:double_integrator_w}}
\vspace*{-30pt}
\end{figure}

\subsection{Ground Vehicle Model}

The ``Dubin's car'' \cite{Dubins57} is a popular model for autonomous ground and air vehicles and has been employed in a wide variety of applications \cite{Bhatia08, Chen07a, Gray12}. Its dynamics are:
\begin{equation}
 \dot{a} = v\cos(\theta), \quad \dot{b} = v\sin(\theta), \quad \dot{\theta} = \omega,
\end{equation}
where the states are the x-position ($a$), y-position ($b$) and yaw angle ($\theta$) of the vehicle and the control inputs are the forward speed ($v$) and turning rate ($\omega$). A change of coordinates can be applied to this system in order to make the dynamics polynomial \cite{DeVon07}. The rewritten dynamics are given by:
\begin{equation}
  \dot x_1 = u_1, \quad \dot x_2 = u_2, \quad \dot x_3 = x_2u_1 - x_1u_2.
\end{equation}
This system is also known as the Brockett integrator and is a popular benchmark since it is prototypical of many nonholonomic systems. Notice that the system has an uncontrollable linearization and does not admit a smooth time-invariant control law that makes the origin asymptotically stable \cite{DeVon07}. Hence, this example illustrates the advantage of our method when compared to linear control synthesis techniques. We solve the ``free final time'' problem to construct a time-varying control law that drives the initial conditions in $X = \{ x \  | \ \|x\|^2 \leq 4 \}$ to the target set $X_T = \{
x \ | \ \|x\|^2 \leq 0.1^2 \}$ by time $T=4$. In the Dubin's car coordinates, the target set is a neighborhood of the origin while being oriented in the positive a-direction. The control is restricted
to $u_1, u_2 \in [-1,1]$.  Figure~\ref{fig:brockett} plots outer approximations of the BRS for $k =
5$. Figure~\ref{fig:samp_brockett} illustrates two sample trajectories
generated using a feedback controller designed by our algorithm after
transforming back to the original coordinate system. 
Solving the SDP took $599$ seconds.

\begin{figure}
 \centering
   \subfigure[$k=5,$ $(x_1,x_3)$ plane\label{fig:brockett_1v3}]{\includegraphics[width=0.49\columnwidth]{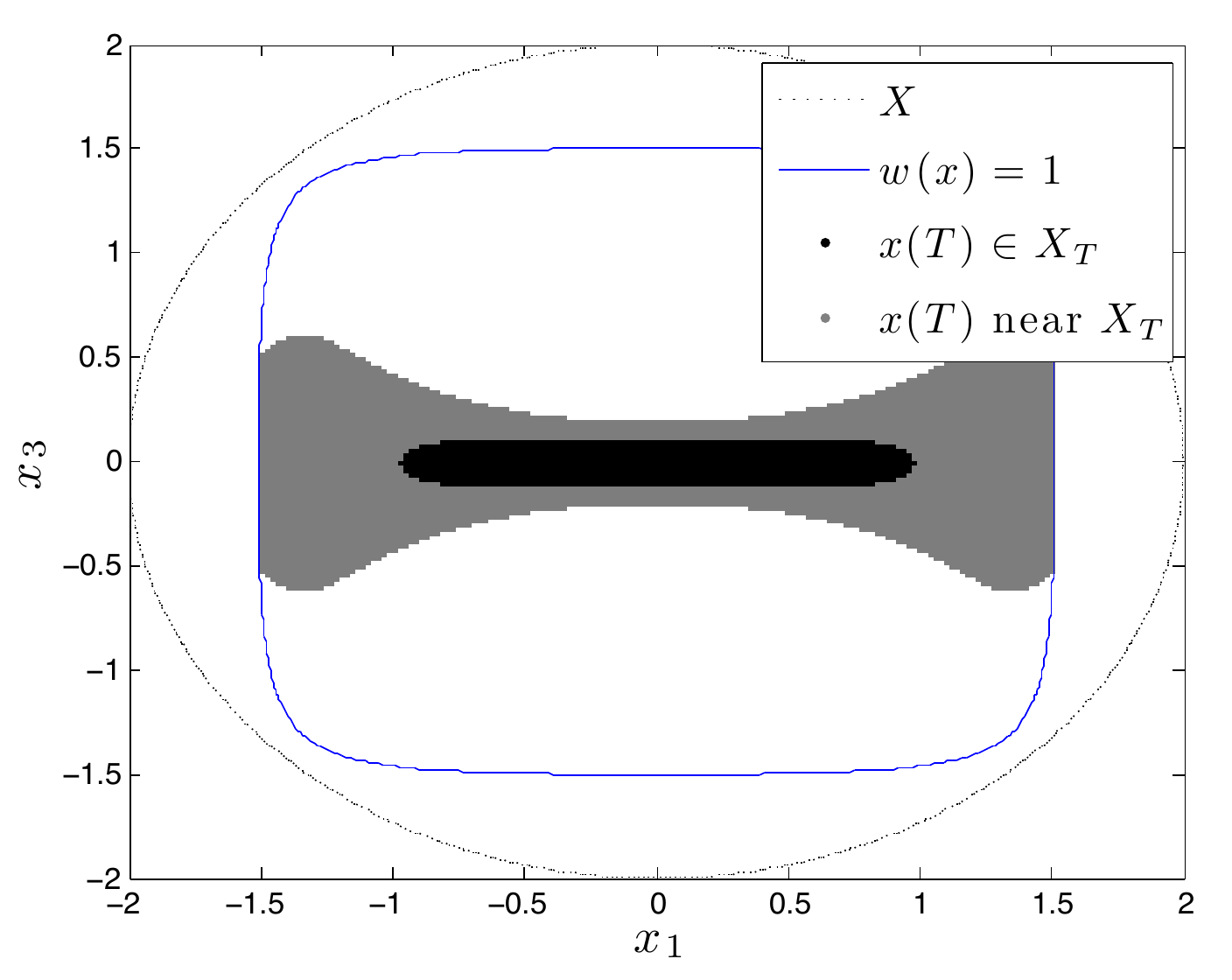}}
   \subfigure[$k=5,$ $(x_1,x_2)$ plane\label{fig:brockett_1v2}]{\includegraphics[width=0.49\columnwidth]{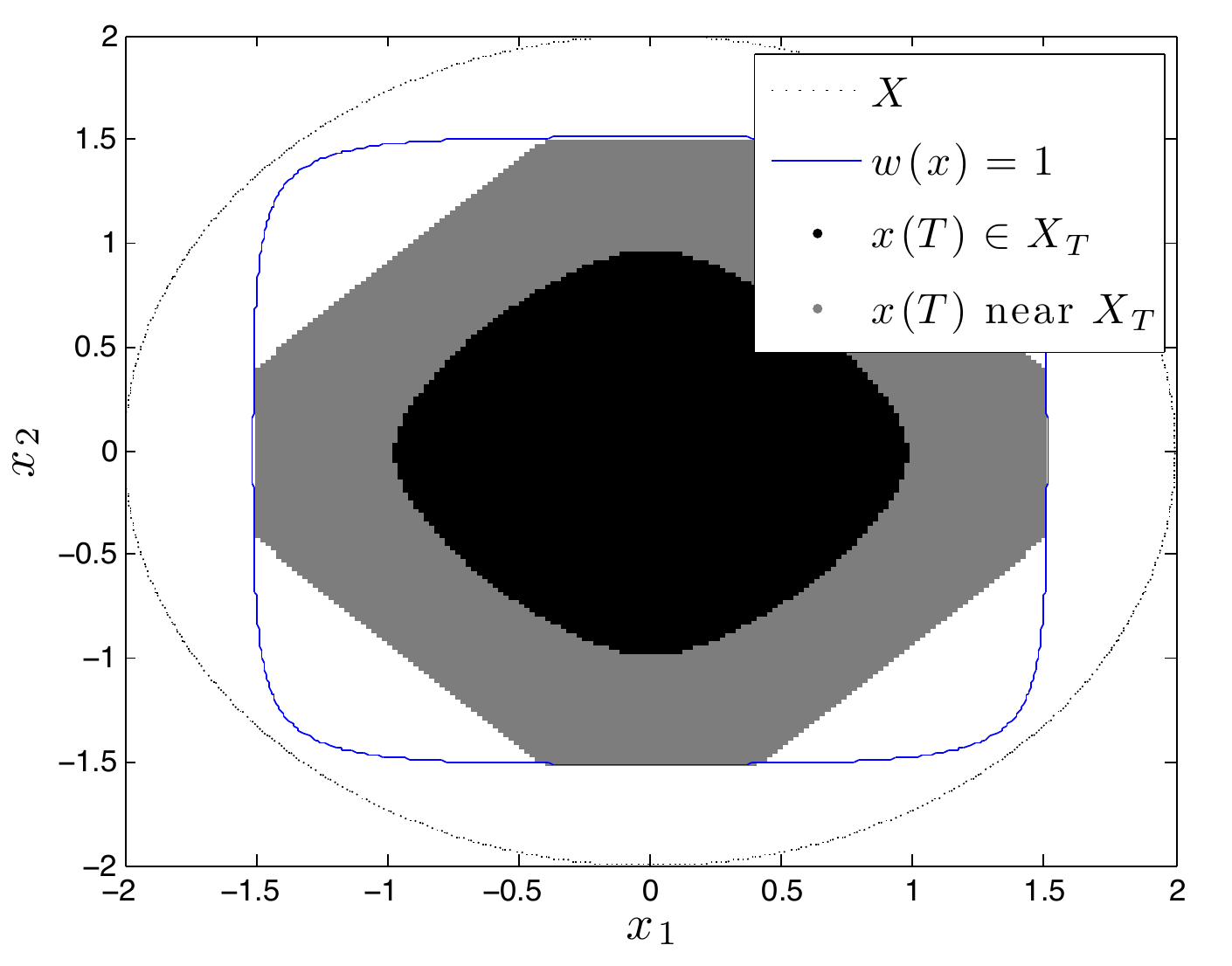}}
	\vspace*{-7pt}
    \caption{The boundary of $X$ (black line) and the outer approximation of the BRS (blue line) are shown in the $(x_1,x_3)$ and $(x_1,x_2)$ planes. In each plane, black points indicate initial conditions of controlled solutions with $x(T) \in X_T$ (i.e. $\|x(T)\|^2 \leq 0.1^2$), and grey points indicate initial conditions of solutions ending near the target set (specifically $\|x(T)\|^2 \leq 0.2^2$). \label{fig:brockett}}
\vspace*{-10pt}
\end{figure}

\begin{figure}
 \centering
   \includegraphics[width=0.8\columnwidth]{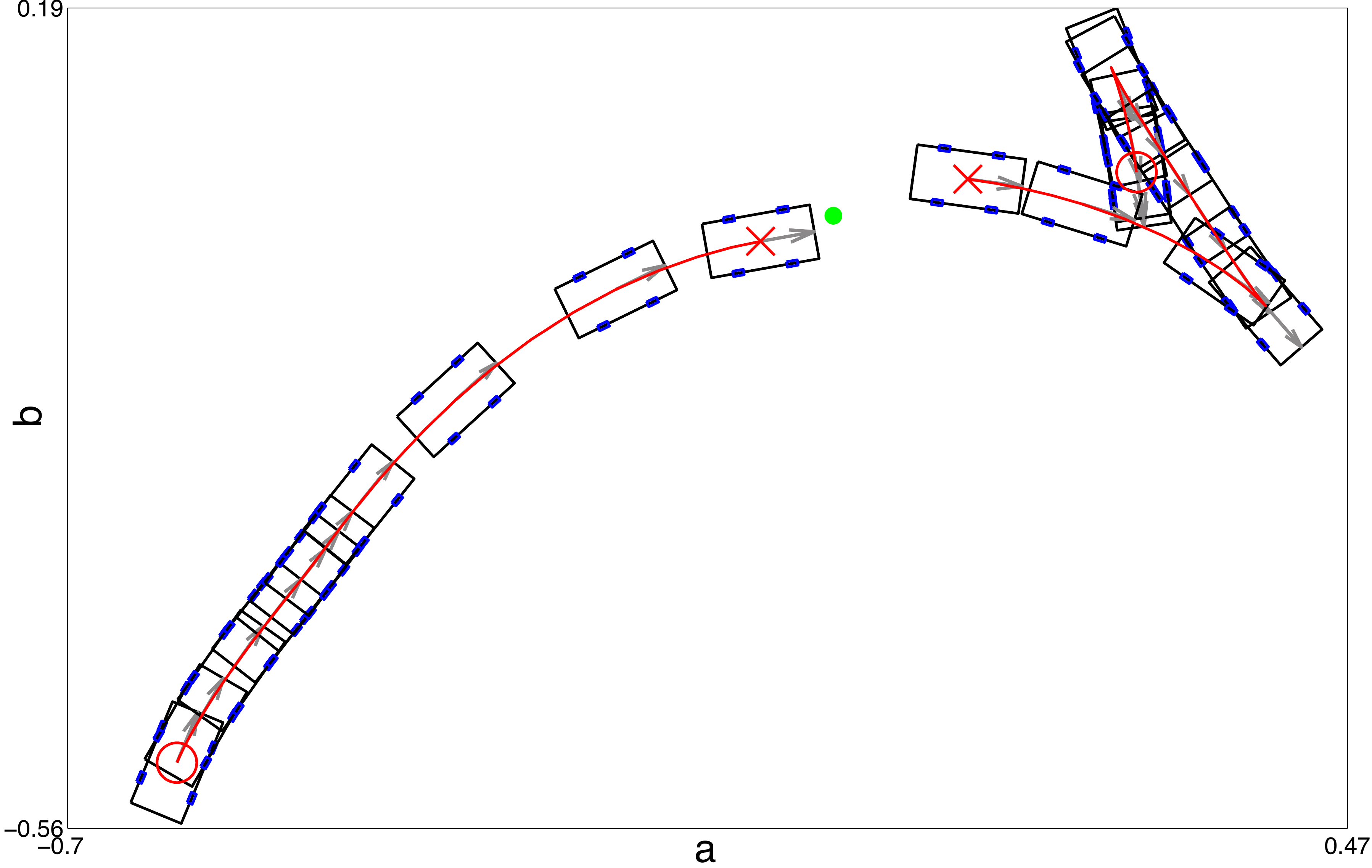}
    \caption{A pair of sample trajectories drawn in red generated by
      our algorithm for the Dubin's car system. Each drawn time sample
      of the car is colored black with blue--colored tires and grey
      forward--facing direction. The origin is marked by a green dot. Each trajectory is initialized at the
      `O' mark and terminates at the state where the target set is first
      reached, which is marked with an the `X' mark.  The trajectory
      starting in the upper right hand corner executes a three-point
      turn to arrive at the desired position and orientation.
    \label{fig:samp_brockett}}
\vspace*{-20pt}
\end{figure}

\subsection{Torque Limited Simple Pendulum}
Next, we consider the torque limited simple pendulum, described by the equations
\begin{equation}
  \dot x_1 = x_2, \qquad I \dot x_2 = mgl\sin(x_1) - b x_2 + u,
\end{equation}
where $x_1$ represents the angle $\theta$ from upright, $x_2$
represents the
angular rate $\dot \theta$, and $u$ represents a torque source at the
pivot constrained to take values in $U = [-3,3]$.
We take $m = 1$, $l
= 0.5$, $I = ml^2$,  $b = 0.1$ and $g =
9.8$.  The bounding set is defined by  $x_1 \in [-\pi,\pi)$ and
$x_2 \in [-8,8]$.
Our method can be readily adapted to handle dynamics with
trigonometric dependence on a state, $x$, so long as the dynamics are polynomial in
$\sin(x)$ and $\cos(x)$.  This is accomplished by introducing
indeterminates $c$ and $s$ identified with $\sin(x)$ and $\cos(x)$
and modifying the approach to work over the quotient ring associated
with the equation $1 = c^2+s^2$ \cite{Parrilo03}.  


For this example, we solve the ``free final time'' problem by taking $T = 1.5$, and defining the target set
as $X_T = \{ (x_1,x_2) \; | \; \cos(x_1)
\geq 0.95, x_2^2 \leq 0.05 \}$.  The running time for the SDP is $11$ mins $20$ secs for $k=5$. Figure~\ref{fig:pendulum} plots sample solutions and summarizes the initial conditions that reach the
target set. Notice that the controller is able to ``swing-up''
states close to the downright position to the upright configuration
despite the stringent actuator limits and a short time-horizon.  

\begin{figure}
  \centering
  \includegraphics[width=0.8\columnwidth]{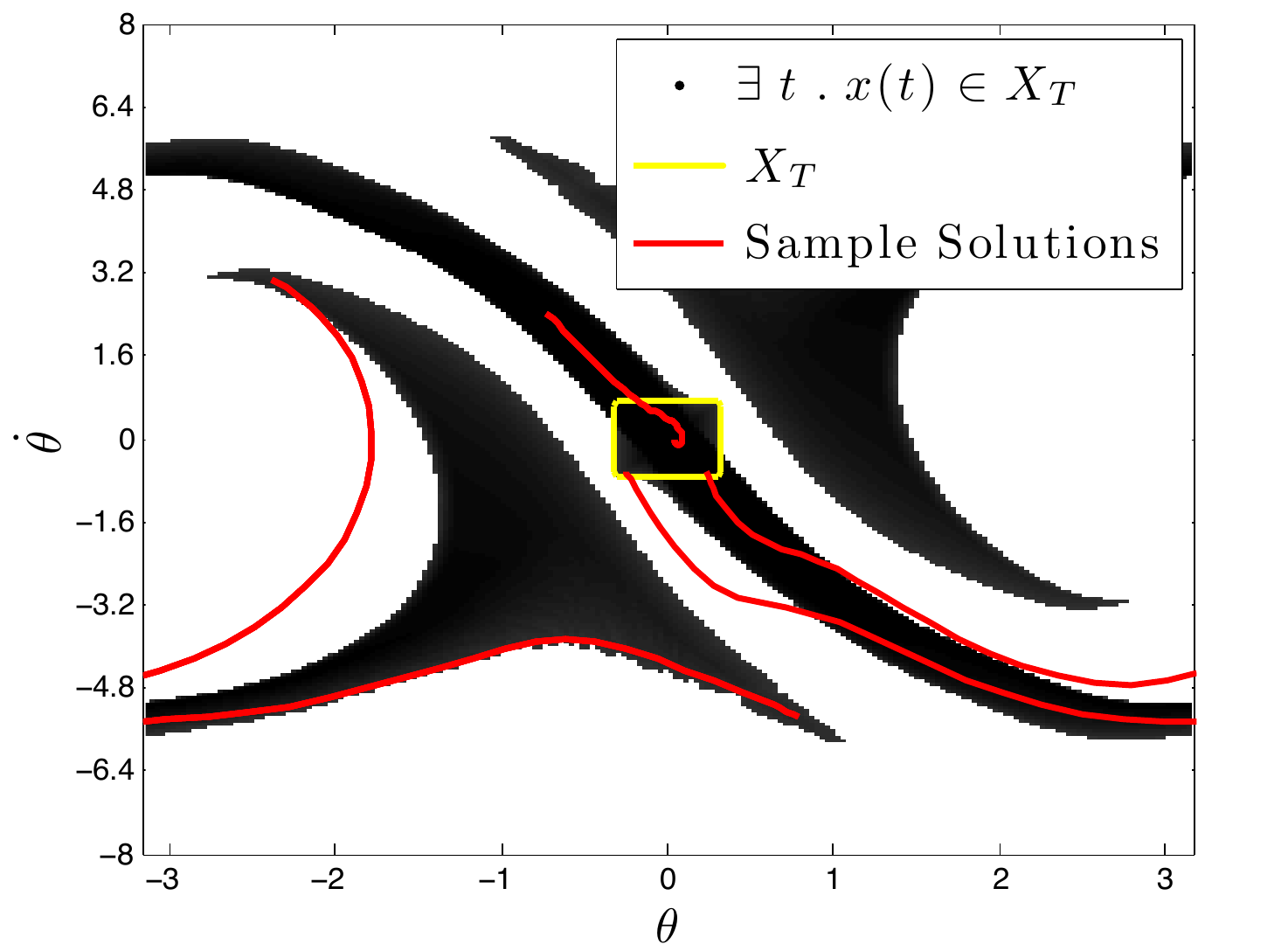}
\vspace*{-5pt}
  \caption{ A depiction of the controller performance for $k=5$ for the torque limited 
    simple pendulum. Black regions indicate sampled initial conditions whose controlled
    solutions pass through the target set (yellow square).  Three
    sample solutions are also plotted (red) each with terminal
    conditions in the target set. 
    Note the solution starting near
    $(-2,3.2)$ passes through zero velocity | the solution ``pumps''
    to reach the upright position.\label{fig:pendulum}
  }
 \vspace{-20pt}
\end{figure}

\subsection{Planar Quadrotor}

Finally, we demonstrate the scalability of our approach on a six state, two input planarized quadrotor model used in a various robotic applications \cite{Hoffmann07,Lupashin10}. The dynamics are defined by \cite{Steinhardt11a}:
\begin{equation}
\begin{aligned}
\ddot{x}_1 &= -(u_1 + u_2) \sin(\theta)/m \\
\ddot{x}_2 &= -g + (u_1 + u_2)\cos(\theta)/m \\
\ddot{\theta} &= L(u_2 - u_1)/I
\end{aligned}
\end{equation}
where $x_1,x_2,$ and $\theta$ are the horizontal and vertical positions, and the attitude of the quadrotor, respectively. The control inputs $u_1$ and $u_2$ are the force produced by the left and right rotors, respectively, and are bounded to have a thrust to weight ratio of $2.5$. Further, $L = 0.25$ is the length of the rotor arm, $m = 0.486$ is the mass, $I = 0.00383$ is the moment of inertia and $g = 9.8$ is the acceleration due to gravity. Using a time horizon of $T = 4$, we solve a ``free final time'' problem and require trajectories reach the target set $X_T = \{ x \ | \ \|x\|^2 \leq 0.1\}$. The bounding set is $X = \{ x \ | \ \|x\|^2 \leq 1\}$. We apply the proposed control design method with $k = 2$, and handle trigonometric terms in the same manner as the pendulum example. The SDP takes $49$ minutes to solve. The resulting controller is able to stabilize a large set of initial conditions. Figure \ref{fig:quad_plot} illustrates a few representative trajectories of the closed-loop system.

\begin{figure}
 \centering
   \includegraphics[width=0.95\columnwidth]{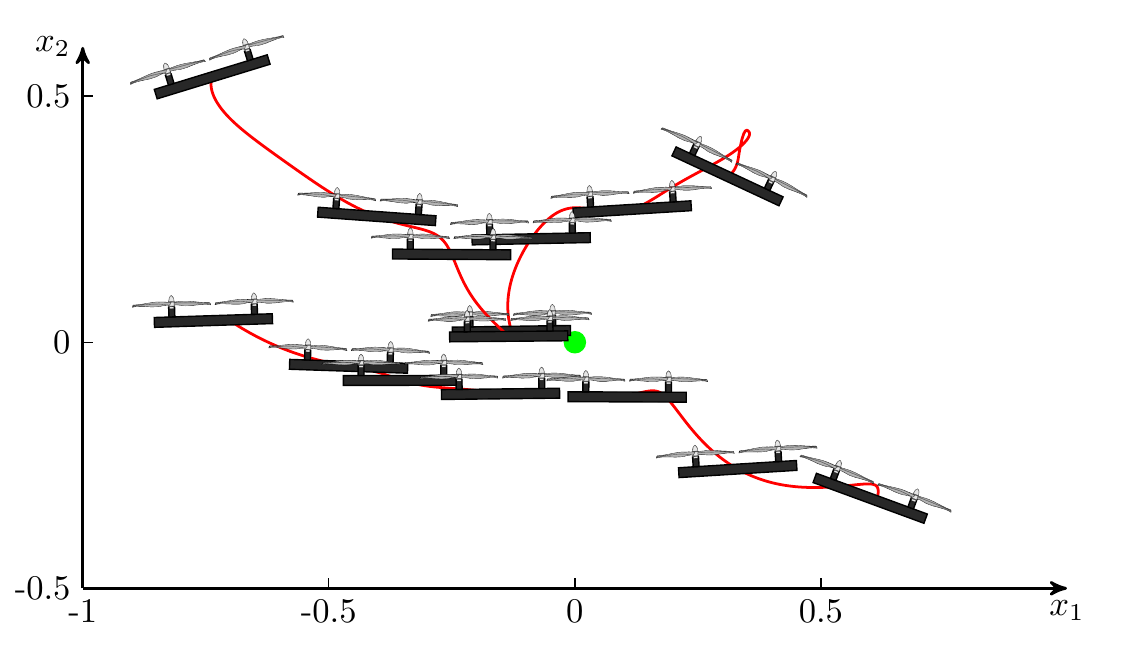}
\vspace*{-5pt}
    \caption{Four trajectories drawn in red generated by
      our algorithm for the planar quadrotor. Each drawn time sample
      of the car is colored black with props on the upward--facing direction. The origin is marked by a green dot.
    \label{fig:quad_plot}}
\end{figure}

\subsection{Satellite Attitude Control}
Finally, we demonstrate the scalability of our approach on a more complicated $6$ state system with $3$ inputs
describing attitude control of a satellite with thrusters applying
torques.  The
dynamics are defined by 
\begin{flalign*}
  H \dot \omega =\: -\Omega(\omega)H\omega + u, \
  \dot \psi = \: \frac{1}2 (I+\Omega(\psi)+\psi\psi^T)\omega,
\end{flalign*}
where $\omega \in \RR^3$ are the angular velocities in the body-frame,
$\psi \in \RR^3$ represent the attitude as  modified Rodriguez parameters (see
\cite{Prajna04a}), $\Omega: \RR^3 \to \RR^{3 \times 3} $ is the
matrix defined so that $\Omega(\psi)\omega = $ is the cross
product $\psi  \times \omega$,
  and $H\in \RR^{3\times 3}$ is the inertia matrix.  We let $H$ be
  diagonal with
  $[H]_{11}=2,$ $[H]_{22}=1$ and $[H]_{33} = \frac{1}2$.
  
  We take the input constraint set as $U = [-1,1]^3$ and the origin as a
  target set.  
We apply the proposed control design methods with
  $k=3$.  Solving the SDP took approximately $6$ hours. 
 Figure~\ref{fig:satellite} examines the controller performance.  A
 set of initial conditions are sampled from a hyperplane, and those
 whose solutions arrive near the target set are highlighted. We note that a SDP with $k=2$ takes only $5$ minutes, but yields a controller and BRS approximations that are slightly inferior, but still useful in practice.

\begin{figure}
  \centering
  \includegraphics[width=0.95\columnwidth]{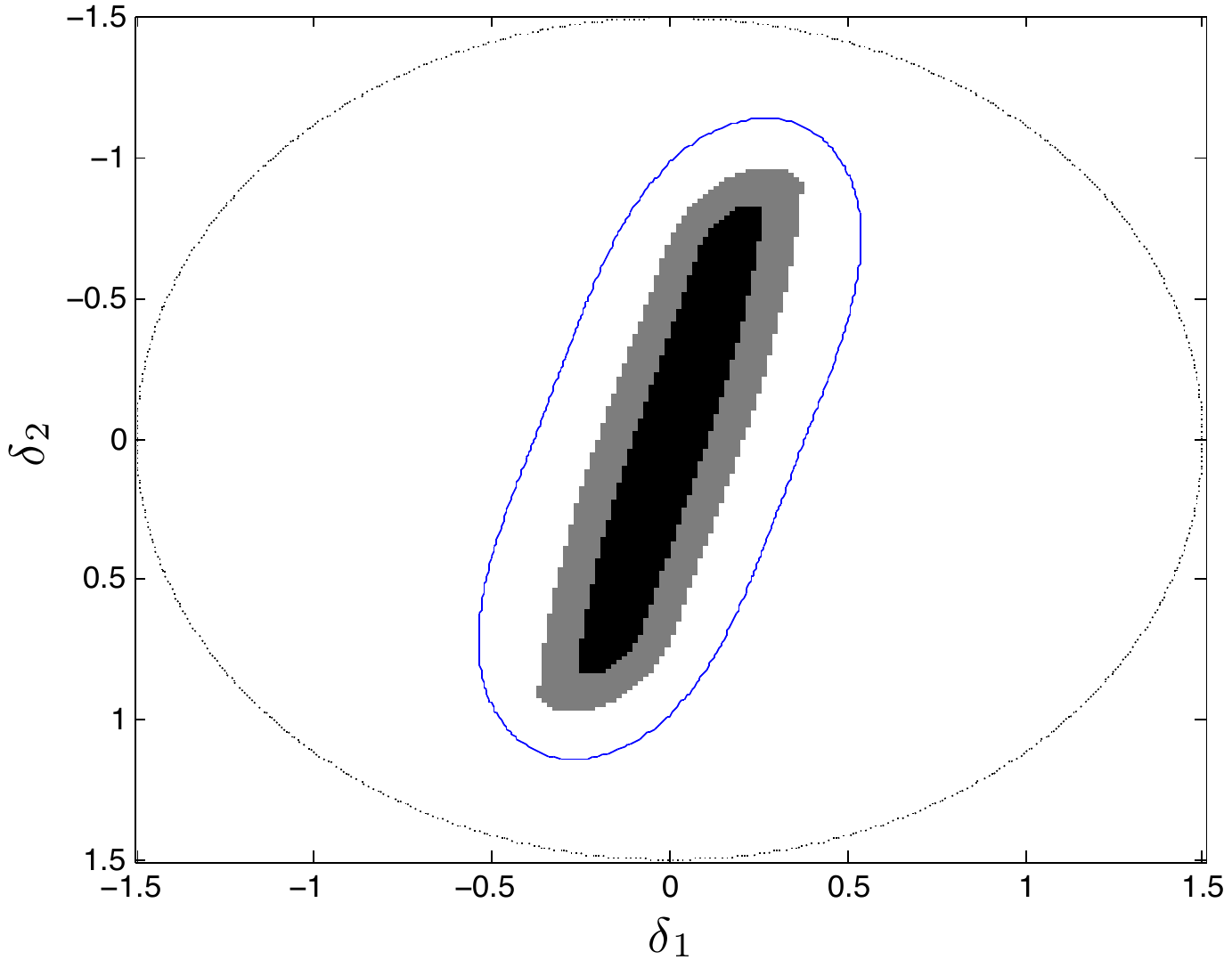}
  \caption{
    Demonstration of the controller performance for $k=3$.  Points are
    sampled from the bounding set (dashed black) and the set
    where $w(x) \leq 1$ (boundary in blue).  To excite coupled
    dynamics between the angular velocities, initial conditions are chosen from the hyperplane with
    coordinates $(\delta_1,\delta_2)$ given by $\delta_1 =
    (\psi_1+\psi_2)/\sqrt{2}$ and $\delta_2 = (\dot \omega_1 +
    \dot \omega_2)/\sqrt{2}$.  Black (resp. grey) points indicate initial conditions
    whose controller solution satisfies $\|x(T)\| \leq 0.1$
    (resp. $\|x(T)\| \leq 0.2$).
    \label{fig:satellite} 
  }
\end{figure}

\section{Conclusion}
\label{sec:conclusion}

We presented an approach for designing feedback controllers that maximize the size of the BRS by posing an infinite dimensional LP over the space of non-negative measures.
Finite dimensional approximations to this LP in terms of SDPs can be used to obtain outer approximations of the largest achievable BRS and polynomial control laws that approximate the optimal control law. In contrast to previous approaches relying on Lyapunov's stability criteria, our method is inherently convex and does not require feasible initialization.
The proposed method can be used to augment existing feedback motion planning techniques that rely on sequencing together BRSs in order to drive some desired set of initial conditions to a given target set.
The number of distinct controllers required by such algorithms could be significantly reduced (potentially down to a single feedback law) by using our algorithm. By reasoning about the nonlinear dynamics of a robotic system, our algorithm should also be able to obtain improved performance during dynamic tasks while maintaining robustness. 

We are presently pursuing convergence results that guarantee the set-wise convergence of the BRS of the controllers generated via \eqref{eq:controllers} to the largest achievable BRS, which is stronger than the result in Theorem \ref{thm:controller convergence}.
We are also working to extend our method to hybrid dynamical systems. 
Our approach potentially can address the difficulties that linearization based approaches face due to the inherent nonlinearities associated with hybrid systems such as walking robots. 
\section*{Acknowledgements}
The authors are grateful to Milan Korda for many helpful discussions.  
This work was supported by ONR MURI grant N00014-09-1-1051, NSF Contract IIS-1161679 and the Siebel Scholars Foundation. 


\bibliographystyle{abbrv} 
\bibliography{elib}

\end{document}